\documentclass{article}

\usepackage{booktabs}

\usepackage{times}

\usepackage{amsmath}
\usepackage{amssymb}
\usepackage{amsthm}

\usepackage{amsfonts}
\usepackage{caption}

\usepackage{thmtools}
\usepackage{thm-restate}

\usepackage{bbold}

\usepackage{latexsym}
\usepackage{graphics}
\usepackage{graphicx}
\usepackage{epstopdf}
\usepackage{caption}
\usepackage{subcaption}

\usepackage{algorithm}
\usepackage[noend]{algorithmic}

\usepackage[colorinlistoftodos,color=blue!30!white]{todonotes}

\usepackage{url}
\usepackage{hyperref}

\usepackage[capitalize,noabbrev]{cleveref}

\newcommand{\approach}[1]{\emph{#1}}
\newcommand{\alg}{{\tt WaterfallUCB1}}

\newtheorem{theorem}{Theorem}

\title{Waterfall Bandits: Learning to Sell Ads Online}

\author{Branislav Kveton \thanks{Authors are listed in alphabetical order.} \\ Google Research \and Saied Mahdian\\ Stanford University \and S. Muthukrishnan\\ Rutgers University \and Zheng Wen\\ Adobe Research \and Yikun Xian \\ Rutgers University}
\date{}

\begin{document}

\maketitle

\begin{abstract}
A popular approach to selling online advertising is by a \emph{waterfall},
where a publisher makes sequential price offers to ad networks for an inventory,
and chooses the winner in that order. The publisher picks the order and prices to maximize her
revenue. A traditional solution is to learn the demand model and
then subsequently solve the optimization problem for the given demand model. This will
incur a linear regret. We design an online learning algorithm for solving this problem,
which interleaves learning and optimization, and prove that this algorithm
has sublinear regret. We evaluate the algorithm on both synthetic and real-world
data, and show that it quickly learns high-quality pricing strategies. This is the
first principled study of learning a waterfall design online by sequential experimentation.
\end{abstract}

\section{Introduction}
\label{sec:introduction}

Online publishers  typically generate revenue by placing advertisements. For example, when a user visits a webpage, there are locations called {\em slots} each of which may have an {\em impression} of an advertisement (ad).  
\begin{itemize}
\item[$\bullet$]
A slot may be sold directly to a specific brand advertiser. In that case, when a user arrives at the webpage, the publisher calls the advertiser and places the ad returned as the impression. 
\item[$\bullet$]
A slot may be sold via third parties such as Google's DoubleClick Ad Exchange. In this case, when a user arrives at the webpage, the publisher calls the ad exchange which in turn calls many intermediaries called {\em ad networks}. Each ad network has several advertisers as its customers and bids on behalf of one of its chosen customers. The ad exchange runs an auction among the bids and returns the winner to the publisher which becomes the ad impression for the user. 
\item[$\bullet$]
A slot may be sold directly to different ad networks. In this case, publishers typically run what is a called a {\em waterfall}. 
In the waterfall, the publisher chooses a permutation of the ad networks.
The publisher calls each ad network sequentially according to the permutation and offers a
price. The ad network has to bid above that price to win the opportunity to place the ad at that slot.
If the ad network does not make an adequate bid, the slot is offered to the next ad network and so on. 
The publisher gets to choose the permutation and reserve prices.
\end{itemize}

The three methods above 
trade off between control, margins and relationships between publishers, ad networks and advertisers. Often publishers combine these methods. For example, publishers might use direct deals for premium slots (like top of homepages), use waterfall variants for torso inventory, and Ad Exchanges for tail or remnant slots that did not get sold by the other methods. Readers who wish more background on the ad business and the role of waterfalls can see \cite{Zawadzinski,Vinnakot2017} or see DoubleClick's support pages\footnote{\url{https://support.google.com/dfp_premium/answer/3007370?hl=en}} \footnote{\url{https://www.sovrn.com/hub/learn/beginners-guide-dfp/}} . 

In this paper, we address the central question how publishers can design the waterfall.
We formalize this problem as learning the optimal order of ad networks with their offered prices. 
Our objective is to maximize the expected revenue of the publisher online in a sequence of $n$ steps, which is equivalent to minimizing the expected regret with respect to the best solution in hindsight. 

\begin{itemize}
\item[$\bullet$]
We formalize and study the problem of publisher learning and optimizing ad revenue from waterfall design 
 as an online learning problem with partial feedback.

\item[$\bullet$]
We develop a bandit style solution and propose a computationally-efficient UCB-like algorithm for this problem, which we call $\alg$.
Our learning problem is challenging for two reasons. 
First, the space of feasible solutions, all permutations of ad networks and their offered prices, is exponentially large. 
Second, our problem suffers from partial feedback, which is similar to that in \emph{cascading bandits} \cite{kveton15cascading,kveton15combinatorial}. 
In particular, if an ad network accepts an offer, the learning agent does not learn if any of the subsequent ad networks would have accepted their offered prices.
\item[$\bullet$]
We prove an upper bound on the expected $n$-step regret of algorithm $\alg$. The upper bound is sublinear in $n$ and polynomial in all other quantities of interest. The key step in our analysis is a new regret decomposition, which is of independent interest beyond our motivating domain of online advertising.
\item[$\bullet$]
We conduct extensive experiments on both synthetic and real-world data, which show that $\alg$ learns high-quality solutions. In addition, we investigate several practical settings that are encountered by publishers: 
\begin{itemize}
 \item[$\circ$] Publishers typically have many web pages with multiple ad slots per page. We show with real-world data that  waterfall learning for all ad slots yields good solutions when the ad networks behave similarly across the ad slots. 
 \item[$\circ$] Instead of going directly to ad networks, publishers can go to third parties that aggregate over ad networks. We show our algorithm $\alg$ can also learn to sell in this setting and it does not overfit.
\end{itemize}
\end{itemize}

Taken together, the above represents the first principled study of publisher revenue when using waterfall to optimize ad placement. 
\vspace*{-1mm}

\section{Selling in the Waterfall}
\label{sec:waterfall}

The problem of selling one ad slot in the waterfall can be formalized as follows.
Let $[K] = \{1, \dots, K\}$ be a set of $K$ ad networks. 
Let $Q = \{q_1, \dots, q_M\}$ be a set of $M$ prices, where $q_i \geq 0$ for all $i \in [M]$. 
We use discrete prices as in \cite{chak2010}; they are a first reasonable approach for the waterfall.
We will discuss this further in Section \ref{sec:analysis}.
Then any instance of our problem can be defined by a tuple $(K, Q, (\mathbb{P}_a)_{a \in [K]})$, where $\mathbb{P}_a$ is a probability distribution over the valuation of ad network $a$. Without loss of generality, we assume that all prices in $Q$ are in $[0, 1]$, and that the support of $\mathbb{P}_a$ is a subset of $[0, 1]$ for all $a \in [K]$. We assume that the valuation of any ad network $a$, $v_a \sim \mathbb{P}_a$, is drawn independently from the valuations of all other ad networks.

The publisher sells to the ad networks as follows. First, it chooses a permutation of the ad networks $(a_1, \dots, a_K) \in \Pi(K)$ and offered price $p_i \in Q$ for each ad network $a_i$, where $\Pi(K)$ is the set of all permutations of $[K]$. Then the publisher contacts the ad networks sequentially, from $a_1$ to $a_K$, and tries to sell the ad slot to them. In particular:
\begin{itemize}
\item[$\bullet$] Ad network $a_1$ is contacted first.

\item[$\bullet$] If ad network $a_i$ is contacted and $p_i \leq v_{a_i}$, the offered price is lower than or equal to the valuation of ad network $a_i$, the offer is accepted. Then the publisher earns $p_i$ and does not contact any of the remaining ad networks.

\item[$\bullet$] If ad network $a_i$ is contacted and $p_i > v_{a_{i}}$, the offered price is higher than the valuation of ad network $a_i$, the offer is rejected. Then the publisher contacts ad network $a_{i + 1}$ if $i < K$. If $i = K$, the publisher does not sell the ad slot and earns zero.
\end{itemize}

We denote by $A = ((a_i)_{i \in [K]}, (p_i)_{i \in [K]})$ the \emph{action} of the publisher. The set of \emph{feasible actions} is $\mathcal{A} = \Pi(K) \times Q^K$. For any ad network $a \in [K]$ and price $p \in Q$, we define \emph{acceptance probability} $\bar{w}(a, p) = P(p \leq v_a)$, the probability that ad network $a$ accepts price $p$ under valuation distribution $\mathbb{P}_a$. We refer to any pair of the ad network and price, $(a, p)$ for $a \in [K]$ and $p \in Q$, as an \emph{item}; and define the set of all items as $E = [K] \times Q$. Note that $|E| = K M$. For any action $A = ((a_i)_{i \in [K]}, (p_i)_{i \in [K]}) \in \mathcal{A}$ and weight function $u: [K] \times Q \to [0, 1]$, we define
\vspace*{-1mm}
\begin{align}
  f(A, u) = \sum_{i = 1}^K \left[\prod_{k = 1}^{i - 1} [1 - u(a_j , p_j)]\right] u(a_i, p_i) \, p_i\,.
  \label{eq:f}
\end{align}
This is the \emph{expected revenue of the publisher} under action $A$ and acceptance probabilities $u$. In particular, assuming the valuations of ad networks are independent, $\prod_{j = 1}^{i - 1} [1 - u(a_j , p_j)]$ is the probability that all ad networks before ad network $a_i$ do not accept their offered prices, which is equal to the probability that $a_i$ is contacted. Moreover, $u(a_i, p_i)$ is the conditional probability of ad network $a_i$ accepting its offered price $p_i$ after it is contacted.
The objective of the publisher is to maximize its expected revenue by choosing $A \in \mathcal{A}$,
\vspace*{-2mm}
\begin{align}
  A^\ast = \arg\max_{A \in \mathcal{A}} f(A, \bar{w})\,.
  \label{eq:waterfall_optimization}
\end{align}

\noindent We refer to $A^\ast$ as the \emph{optimal solution}.

\subsection{Oracles}
\label{sec:oracles}

No polynomial-time algorithm is known for solving all instances of problem (\ref{eq:waterfall_optimization}). However, computationally-efficient approximations exist \cite{chak2010,Chawla2010}. In this work, we consider approximation algorithms $L$ whose inputs are a weight function $u: [K] \times Q \to [0, 1]$, the number of ad networks $K$, and a set of prices $Q$; and the output is $L(u, K, Q) \in \mathcal{A}$. We say that algorithm $L$ is a \emph{$\gamma$-approximation} for $\gamma \in (0,1]$ if $f(L(u, K, Q), u) \geq \gamma \max_{A \in \mathcal{A}} f(A, u)$ for any $u$.

Note that when ad networks are assigned prices, the optimal order of the ad networks is in the descending order of their assigned prices. This follows from the definition of the revenue in \eqref{eq:f}. Since the output of $L$ can be always ordered to satisfy this property, we assume that this property is satisfied without loss of generality. We consider two oracles in this paper, greedy and based on linear programming (LP).

\begin{algorithm}[t]
  \caption{Greedy oracle.}
  \label{alg:alg_greedy}
  \begin{algorithmic}
    \REQUIRE Weight function $u$, number of ad networks $K$, prices $Q$
    \STATE \vspace{-0.1in}
    \FORALL{$a \in [K]$}
      \STATE $p^\ast(a) \gets \arg\max_{p \in Q} u(a, p) \, p$
    \ENDFOR
    \STATE Let $\pi$ be any permutation of $[K]$ such that $p^\ast(\pi(1)) \geq \ldots \geq p^\ast(\pi(K))$
    \STATE $A \gets ((\pi(i))_{i \in [K]}, (p^\ast(\pi(i)))_{i \in [K]})$
    \STATE \vspace{-0.1in}
    \ENSURE Publisher action $A$
  \end{algorithmic}
\end{algorithm}

The pseudocode of the \emph{greedy oracle} is in \cref{alg:alg_greedy}. The oracle has two main stages. First, it assigns to each ad network $a \in [K]$ the price that maximizes the expected revenue of that ad network conditioned on being contacted, $p^\ast(a)$. Second, it orders the ad networks in the descending order of their assigned prices. This oracle is easy to implement and performs well in our experiments. It does not have any approximation guarantee though.

\begin{algorithm}[t]
  \caption{LP oracle.}
  \label{alg:alg_lp}
  \begin{algorithmic}
    \REQUIRE Weight function $u$, number of ad networks $K$, prices $Q$
    \STATE \vspace{-0.1in}
    \STATE Solve 
    \begin{align*}
      \displaystyle
      \max \quad & \sum_{a = 1}^K \sum_{p \in Q} p \, y_{a, p} \\
      \text{s.t.} \quad & \forall a \in [K], p \in Q:
      y_{a, p} \leq u(a, p) \, x_{a, p}\,, \ \
      x_{a, p} \geq 0\,, \ \
      y_{a, p} \geq 0\,; \\
      & \sum_{a = 1}^K \sum_{p \in Q} y_{a, p} \leq 1\,; \quad
      \forall a \in [K]: \sum_{p \in Q} x_{a, p} \leq 1;
    \end{align*}
    \vspace{-0.1in}
    \STATE Let $\tau^\ast$ be the dual variable associated with the last constraint of the above LP \vspace*{-2mm}
    \STATE
    \FORALL{$a \in [K]$}
      \STATE $p^\ast(a) \gets \arg\max_{p \in Q} u(a, p) \, (p - \tau^\ast)$
      \IF{$(\max_{p \in Q} u(a, p) \, (p - \tau^\ast) = 0)$ and $(\tau^\ast \in Q)$}
        \STATE $p^\ast(a) \gets \tau^\ast$
      \ENDIF
    \ENDFOR
    \STATE Let $\pi$ be any permutation of $[K]$ such that $p^\ast(\pi(1)) \geq \ldots \geq p^\ast(\pi(K))$
    \STATE $A \gets ((\pi(i))_{i \in [K]}, (p^\ast(\pi(i)))_{i \in [K]})$
    \STATE \vspace{-0.1in}
    \ENSURE Publisher action $A$
  \end{algorithmic}
\end{algorithm}

The pseudocode of the \emph{LP oracle} is in \cref{alg:alg_lp}. The oracle is based on linear programming and is a \emph{$(1-\frac{1}{e})$-approximation} algorithm \cite{chak2010}. The oracle has three main stages. First, it solves an LP to obtain the value of the dual variable corresponding to the last constraint $\tau^\ast$. Second, it assigns to each ad network $a \in [K]$ the price $p$ that maximizes $u(a, p) \, (p - \tau^\ast)$. This price is denoted by $p^\ast(a)$. Finally, it orders the ad networks in the descending order of their assigned prices.

The optimized variables in the linear program are $x_{a, p}$ and $y_{a, p}$, for $a\in [K]$ and $p \in Q$. The variable $x_{a, p}$ represents the probability that ad network $a$ is offered price $p$. The variable $y_{a, p}$ represents the joint probability that ad network $a$ is offered price $p$ and accepts. The objective is the expected return. The constraints guarantee that the probabilities are consistent and non-negative.

Both discussed oracles can find high-quality strategies for selling a single ad slot in the waterfall. In practice, publishers may be interested in maximizing the revenue from all of their many ad slots. We return to these practical issues in \cref{sec:experiments}.
\vspace*{-2mm}

\section{Waterfall Bandit}
\label{sect:online_alg}

As we discussed earlier, publishers often do not know the valuation distributions of ad networks in advance. However, since they repeatedly sell ad slots to the ad networks, they can learn it. This motivates our study of the waterfall as a \emph{multi-armed bandit (MAB)} \cite{lai85asymptotically,auer02finitetime}, which we call a waterfall bandit. Formally, the \emph{waterfall bandit} is a tuple $(K, Q, (\mathbb{P}_a)_{a \in [K]})$, where the valuation distributions of ad networks $(\mathbb{P}_a)_{a \in [K]}$ are unknown to the publisher. Let $v_{t, a} \sim \mathbb{P}_a$ be the stochastic valuation of ad network $a$ at time $t$. We assume that $v_{t, a}$ is drawn independently from $\mathbb{P}_a$, both across ad networks and in time.

The publisher repeatedly sells to ad networks for $n$ times. At each time $t$, based on past observations, the publisher adaptively chooses \emph{action} $A_t = ((a_{t, i})_{i \in [K]}, (p_{t, i})_{i \in [K]})$, where $a_{t, i} \in [K]$ and $p_{t, i} \in Q$ are the $i$-th contacted ad network at time $t$ and its assigned price, respectively. The publisher receives \emph{feedback} $B_t \in [K] \cup \{\infty\}$, which is the index of the first ad network that accepts its offered price. In particular, when $B_t = i$ for $i \in [K]$, ad network $a_{t, i}$ accepts its offered price $p_{t, i}$ and the \emph{reward} of the publisher is $p_{t, i}$. On the other hand, when $B_t = \infty$, no ad network accepts its offered price and the \emph{reward} of the publisher is zero. Because the ad networks are contacted sequentially, the publisher knows that the offered price $p_{t, i}$ is not accepted by any ad network $a_{t, i}$ such that $i < B_t$ for $i \in [K]$. In summary, the publisher observes responses from all ad networks $a_{t, i}$ such that $i \leq B_t$ for $i \in [K]$, and we refer to these ad networks as being \emph{observed}.
 
We evaluate the performance in the waterfall bandit by a form of regret, where the cumulative reward of the optimal solution $A^\ast$ is weighted by a factor of $\gamma$. In particular, the \emph{scaled $n$-step regret} is defined as
\vspace*{-3mm}
\begin{align}
  R^\gamma(n) = n \gamma f(A^\ast, \bar{w}) - \mathbb{E}\left[\sum_{t = 1}^n f(A_t, \bar{w})\right]\,,
  \label{eq:regret}
\end{align}
where $\gamma > 0$ is the aforementioned scaling factor and $A^\ast$ is the optimal solution in \eqref{eq:waterfall_optimization}. The reason for the scaling factor is that no polynomial-time algorithms exists for solving our offline optimization problem (\cref{sec:oracles}). Therefore, it is unreasonable to assume that we can learn such solutions online, and it is reasonable to compete with the best offline $\gamma$-approximation. Note that the scaled $n$-step regret reduces to the standard $n$-step regret when $\gamma = 1$.

\noindent {\bf Naive Solutions.}
The waterfall bandit can be solved as a multi-armed bandit problem where the expected revenue of each action $A \in \mathcal{A}$ is estimated separately. This solution would not be statistically efficient. The reason is that the number of actions is $|\mathcal{A}| = K! M^K$, and so a naive unstructured solution would have exponential regret in $K$.

The key structure in our learning problem is that the publisher receives feedback on individual ad networks in each action. This setting is reminiscent of stochastic combinatorial semi-bandits \cite{gai12combinatorial,chen14combinatorial,kveton15tight,wen15efficient}, which can be solved statistically efficiently. The challenge is that the publisher may not receive feedback on all ad networks. More specifically, when $a_{t, i} > B_t$ at time $t$, the publisher does not know if ad network $a_{t, i}$ would accepted price $p_{t, i}$ if it was offered that price. Therefore, our problem cannot be formulated and solved as a stochastic combinatorial semi-bandit.

A similar form of partial feedback was studied in cascading bandits \cite{kveton15cascading}, where the learning agent receives feedback on a ranked list of items, for all items in the list up to the first clicked item. The difference in our setting is that Kveton et al. \cite{kveton15cascading} do not consider pricing. Nevertheless, it is reasonable to assume that a similar learning algorithm, which maintains upper confidence bounds on all acceptance probabilities $\bar{w}(a, p)$, for any $a \in [K]$ and $p \in Q$, could solve our problem. We present such an algorithm in \cref{sec:algorithm}.
\vspace*{-1mm}

\section{$\alg$ Algorithm}
\label{sec:algorithm}
\vspace*{-1mm}

In this section, we propose a UCB-like algorithm for the waterfall bandit, which we call $\alg$. The algorithm is presented in \cref{alg:alg1}.

The inputs to $\alg$ are an approximation oracle $L$, the number of ad networks $K$, and a set of prices $Q$. At each time $t$, $\alg$ proceeds as follows. First, it computes an \emph{upper confidence bound (UCB)}
  $U_t(e) = \min \{\hat{w}_{T_{t - 1}(e)}(e) + c_{t - 1, T_{t - 1}(e)}, 1\}$
on the acceptance probability $\bar{w}(e)$ of all $e = (a, p) \in E$, where $E$ is the set of all ad-network and price pairs (\cref{sec:waterfall}), $\hat{w}_s(e)$ is the fraction of accepted offers in $s$ trials when ad network $a$ is offered price $p$, $T_t(e)$ is the number of times that ad network $a$ is offered price $p$ up to time $t$, and $c_{t, s} = \sqrt{(1.5 \log t) / s}$ is the radius of a confidence interval around $\hat{w}_s(e)$ after $t$ steps such that $\bar{w}(e) \in [\hat{w}_s(e) - c_{t, s}, \hat{w}_s(e) + c_{t, s}]$ holds with a high probability. We trim $U_t(e)$ at $1$ so that it can be interpreted as a probability.

After the UCBs are computed, $\alg$ computes its action at time $t$, $A_t = ((a_{t, i})_{i \in [K]}, (p_{t, i})_{i \in [K]})$, using the oracle $L$ and UCBs $U_t$. Then it takes that action and receives feedback $B_t$. Finally, $\alg$ updates its statistics for all \emph{observed} pairs of ad networks and prices, $(a_{t, i}, p_{t, i})$ such that $i \leq \min \{B_t, K\}$.

\begin{algorithm}[t]
  \caption{Algorithm $\alg$.}
  \label{alg:alg1}
  \begin{algorithmic}
  \REQUIRE Oracle $L$, number of ad networks $K$, prices $Q$
  \STATE \vspace{-0.1in}
  \FORALL[Initialization]{$e = (a, p) \in E$}
    \STATE Offer price $p$ to ad network $a$ once
    \STATE $\hat{w}_0(e) \gets \mathbb{1}\{\text{$a$ accepts price $p$}\}, \ T_0(e) \gets 1$
  \ENDFOR
  \STATE
  \FOR{$t = 1, \dots, n$}
    \FORALL[Compute UCBs]{$e = (a, p) \in E$}
      \STATE $U_t(e) = \min \{\hat{w}_{T_{t - 1}(e)}(e) + c_{t - 1, T_{t - 1}(e)}, 1\}$
    \ENDFOR
    \STATE $A_t \gets L(U_t, K, Q)$ \COMMENT{Compute the action using the oracle and $U_t$}
    \STATE Observe $B_t \in \{1, \dots, K, +\infty\}$ \COMMENT{Run the waterfall and get feedback}
    \STATE
    \STATE $T_t(e) \gets T_{t - 1}(e)$, $\forall e \in E$ \COMMENT{Update statistics}
    \FOR{$i = 1, \dots, \min \{B_t, K\}$}
      \STATE $e \gets (a_{t, i}, p_{t, i})$
      \STATE $T_t(e) \gets T_{t}(e) + 1$
      \STATE $\displaystyle \hat{w}_{T_{t}(e)}(e) \gets \frac{T_{t - 1}(e) \hat{w}_{T_{t - 1}(e)}(e) + \mathbb{1}\{B_t = i\}}{T_t(e)}$
    \ENDFOR
  \ENDFOR
  \end{algorithmic}
\end{algorithm}

\vspace*{-1mm}

\section{Analysis}
\label{sec:analysis}

The regret of $\alg$ is bounded in the following theorem.

\begin{theorem}
\label{thm::thm1} Let $\alg$ be run in the waterfall bandit with a $\gamma$-approximation oracle $L$. Then
  $R^\gamma(n) \leq 4 M K \sqrt{1.5 n \log n} + \gamma \frac{\pi^2}{3} M K\,,$
where $K$ is the number of ad networks and $M$ is the number of prices in $Q$.
\end{theorem}

\begin{proof}
We sketch the proof of Theorem~\ref{thm::thm1} below. The detailed proof is in Appendix~\ref{subsub:thm_proof}.

The proof proceeds as follows. First, we show that $f(A, u)$ is monotone in weight function $u$ for any fixed action $A$ (Lemma~\ref{thm::prop1} in Appendix~\ref{subsub:thm_proof}). Second, based on this monotonicity property, we bound the per-step scaled regret at time $t$, $R_t^\gamma = \gamma f(A^\ast, \bar{w}) - f(A_t, \bar{w})$, under ``good event" $\bar{\mathcal{E}}_t$, that all $\bar{w}(e)$ are inside of their confidence intervals at time $t$. This novel regret decomposition is presented in Lemma~\ref{lemma::lemma1} below.

\begin{restatable*}{lemma}{lemmadecomp}
\label{lemma::lemma1} Conditioned on ``good event" $\bar{\mathcal{E}}_t$, the per-step scaled regret at time $t$ is bounded as
  $R_t^\gamma \leq \sum_{i = 1}^K \mathbb{E}_t\left[\mathbb{1}\{G_{d_{t, i}, t}\}\right] \phi_{d_{t, i}, t}\,,$
where $d_{t, i} = (a_{t, i}, p_{t, i})$, $\phi_{e, t} = U_t(e) - \bar{w}(e)$, and $G_{e, t}$ is the event that item $e$ is observed at time $t$.
\end{restatable*}

The proof of Lemma~\ref{lemma::lemma1} is in Appendix~\ref{subsub:thm_proof}. Note that the lemma decomposes the regret at time $t$ into those of observed items. Based on the definition of ``good event" $\bar{\mathcal{E}}_t$ and some algebra, we have
  $\sum_{t = 1}^n \mathbb{E}[\mathbb{1}\{\bar{\mathcal{E}}_t\} R^\gamma_t]
   \leq 2 \sqrt{1.5 \log n} \sum_{e \in E} \sum_{t = 1}^n \frac{1}{\sqrt{t}} 
   \leq 4 M K \sqrt{1.5 \, n \log n}\,.$
On the other hand, we bound the regret under ``bad event" $\mathcal{E}_t$, that at least one $\bar{w}(e)$ is outside of its confidence interval, using Hoeffding's inequality. Specifically, we get
  $\sum_{t = 1}^n \mathbb{E}[\mathbb{1}\{\mathcal{E}_t\} R^\gamma_t] \leq \gamma \frac{\pi^2}{3} M K$.
The bound in Theorem~\ref{thm::thm1} follows directly from combining the above two inequalities. 
\end{proof}

Theorem~\ref{thm::thm1} provides a $O(M K \sqrt{n \log n})$ gap-free upper bound on the scaled $n$-step regret of $\alg$. We discuss the tightness of this bound below. The dependence on $\sqrt{n \log n}$ is standard in gap-free bounds in similar problems \cite{kveton15tight,kveton15combinatorial}, and it is considered $O(\sqrt{\log n})$ from being tight. The dependence $M K$ is expected, since $\alg$ estimates $M K$ values, one for each ad network and price. However, linear dependence on $M K$ may not be tight. We obtain it for two reasons. First, $\alg$ learns $\bar{w}(e)$ separately for each item $e \in E$, and does not exploit any generalization across ad networks and prices. Second, our bound is proved directly from the ``self-normalization" of confidence interval radii (Appendix~\ref{subsub:thm_proof}), not through a gap-dependent bound as in related papers \cite{kveton15tight,kveton15combinatorial}. 

Our analysis also provides a sublinear regret bound with respect to the optimal continuous-price solution. In particular, let all prices be in $[0, 1]$ and suppose the publisher intends to maximize expected revenue up to time $n$. When the prices are discretized on a uniform grid with $M$ points over $[0, 1]$, the maximum instantaneous loss due to discretization is $C K / M$, where $C$ is a problem-specific Lipschitz factor. Under this assumption, the scaled $n$-step regret with respect to the optimal continuous-price solution is bounded by $O(M K \sqrt{n \log n} + \gamma M K) + n C K / M$. Now we choose 
\begin{align*}
M = \sqrt{nC/(\sqrt{n\log n} + \gamma)} 
\end{align*}
and we get a $O\left(K\sqrt{C\cdot n}  \sqrt{\sqrt{n\log n} + \gamma} \right)$ regret bound.

\section{Experiments}
\label{sec:experiments}

In this section, we empirically evaluate the effectiveness of our algorithm.  
We also investigate the settings where our algorithm may be deployed in practice.

\subsection{Methods and Metrics}
\label{sec:methods and metrics}

The input to $\alg$ is an \emph{oracle}, which orders ad networks and assigns prices to them for any model $u$ (\cref{sec:oracles}). We experiment with two oracles, \approach{Greedy} (\cref{alg:alg_greedy}) and \approach{LP} (\cref{alg:alg_lp}).
We compare the following offline and online approaches, where \approach{X} refers to one of the aforementioned oracles:
\begin{enumerate}
\item \approach{Offline-X} is an offline approximation algorithm. The input to the algorithm are all acceptance probabilities, $\bar{w}(a, p)$ for any $a \in [K]$ and $p \in Q$. The probabilities are used by oracle $X$ to order ad networks and assign prices to them. The ordering and prices are computed only once and used in all steps. Although this approach is unrealistic because it assumes that the acceptance probabilities are known, it is a useful baseline for evaluating revenue loss due to not knowing the dynamics of the system.

\item \approach{UCB-X} is the $\alg$ in \cref{alg:alg1}.

\item \approach{Exp$^2$-X} is an online approximation algorithm, which explores in the first $n_0$ steps and then exploits \cite{sutton1998reinforcement}. In the first $n_0$ steps, the algorithm offers random prices to randomly ordered ad networks and collects observations. Then it estimates all acceptance probabilities from its observations. The probabilities are used by oracle $X$ to order ad networks and assign prices to them. The ordering and prices are computed in step $n_0$ and then used in all remaining steps. The exploration parameter $n_0$ tends to be small in practice because random exploration hurts experience.
\end{enumerate}

The performance of all compared algorithms is evaluated by their \emph{expected $n$-step reward},
\begin{align*}
  r(n) = \frac{1}{n} \sum_{t = 1}^n \sum_{i = 1}^K \!
  \left[\prod_{j = 1}^{i - 1} \mathbb{1}\!\left\{p_{t, j} > v_{t, a_{t, j}}\right\}\right] \! 
  \mathbb{1}\!\left\{p_{t, i} \leq v_{t, a_{t, i}}\right\} \, p_{t, i}\,,
\end{align*}
where $A_t = \left((a_{t, i})_{i \in [K]}, (p_{t, i})_{i \in [K]}\right)$ is the action of the publisher at time $t$ and $v_{t, a}$ is the valuation of ad network $a$ at time $t$. We choose this metric instead of the scaled regret in \eqref{eq:regret} because the optimal solution to our offline optimization problem cannot be computed efficiently (\cref{sec:oracles}). The optimal solution is necessary to evaluate \eqref{eq:regret}.

We report the expected reward in \emph{hypothetical dollars} to highlight the business value of our algorithm. 

\subsection{Synthetic Data}
\label{sec:synthetic}

In this experiment, we show that the expected reward of our algorithm approaches that of the best approximation in hindsight. We also demonstrate that our algorithm $\alg$ outperforms \approach{Exp$^2$-X} irrespective of the oracle.

We consider a synthetic problem with a single ad slot and $K = 4$ ad networks. The valuation of each ad network at time $t$ is drawn i.i.d. from beta distribution $\mathrm{Beta}(\alpha, \beta)$, which is parameterized by $\alpha$ and $\beta$. 
As a result, the minimum and maximum valuations of each ad network are $\$0$ and $\$1$ respectively.
The valuation of ad network $1$ is \emph{high}, $v_{t, 1} \sim \mathrm{Beta}(5, 2)$. The valuations of the remaining three ad networks are \emph{low}, $v_{t, a} \sim \mathrm{Beta}(2, 5)$ for any $a \in [4] \setminus \{1\}$. The learning problem is to offer a high price to the ad network $1$, ahead of the other ad networks.

The prices in all algorithms are discretized to $11$ price levels, namely $Q = \{(p - 1) / 10: p \in [11]\}$.
We experiment with both \approach{Greedy} and \approach{LP} oracles. The number of exploration steps in \approach{Exp$^2$-X} is $n_0 = 500$. In this experiment, this setting yields approximately $11$ observations on average for each pair of the ad network and price.

\begin{figure}[t]
  \centering
  \includegraphics[width=1.546in,height=1.35in]{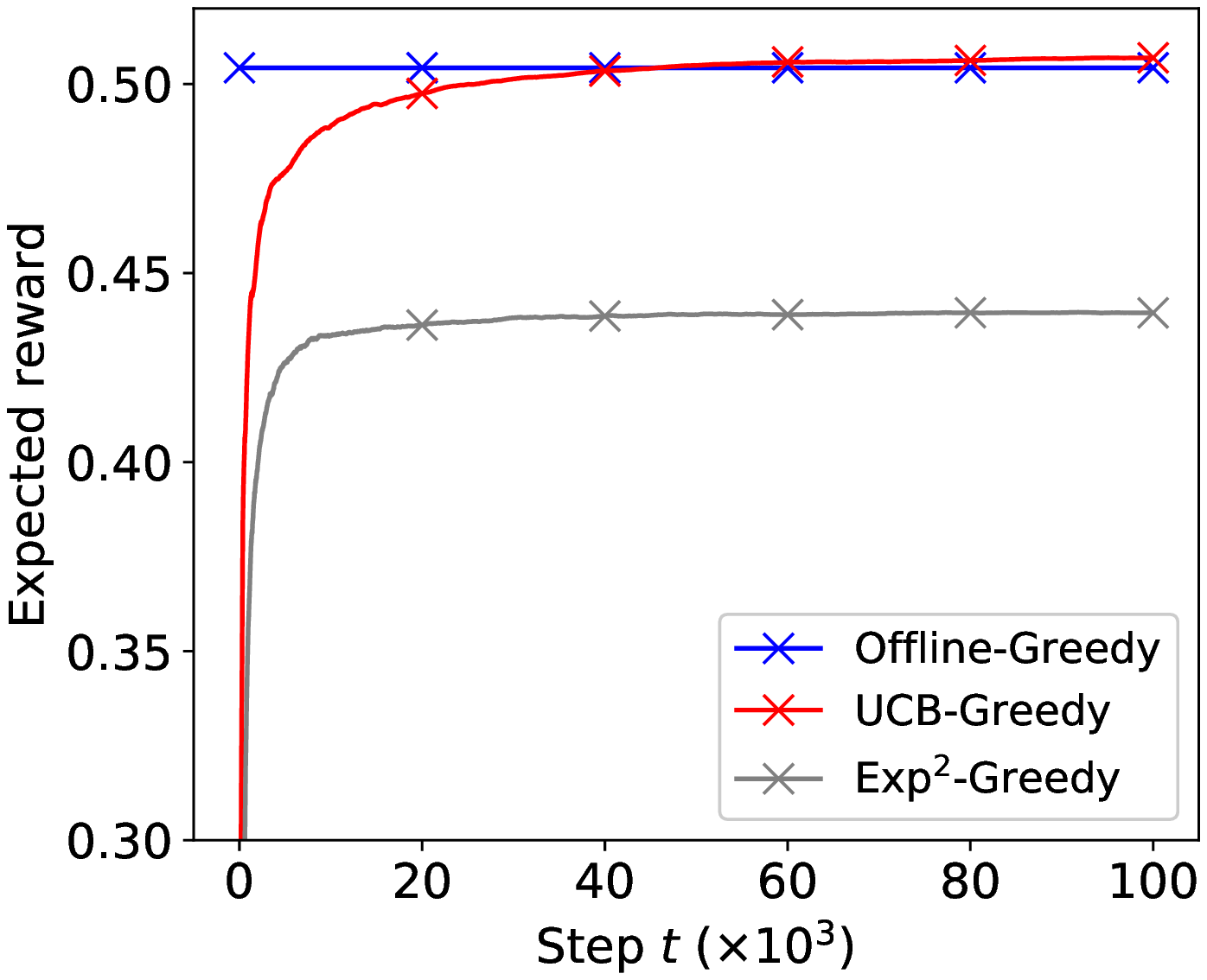}
  \includegraphics[width=1.546in,height=1.35in]{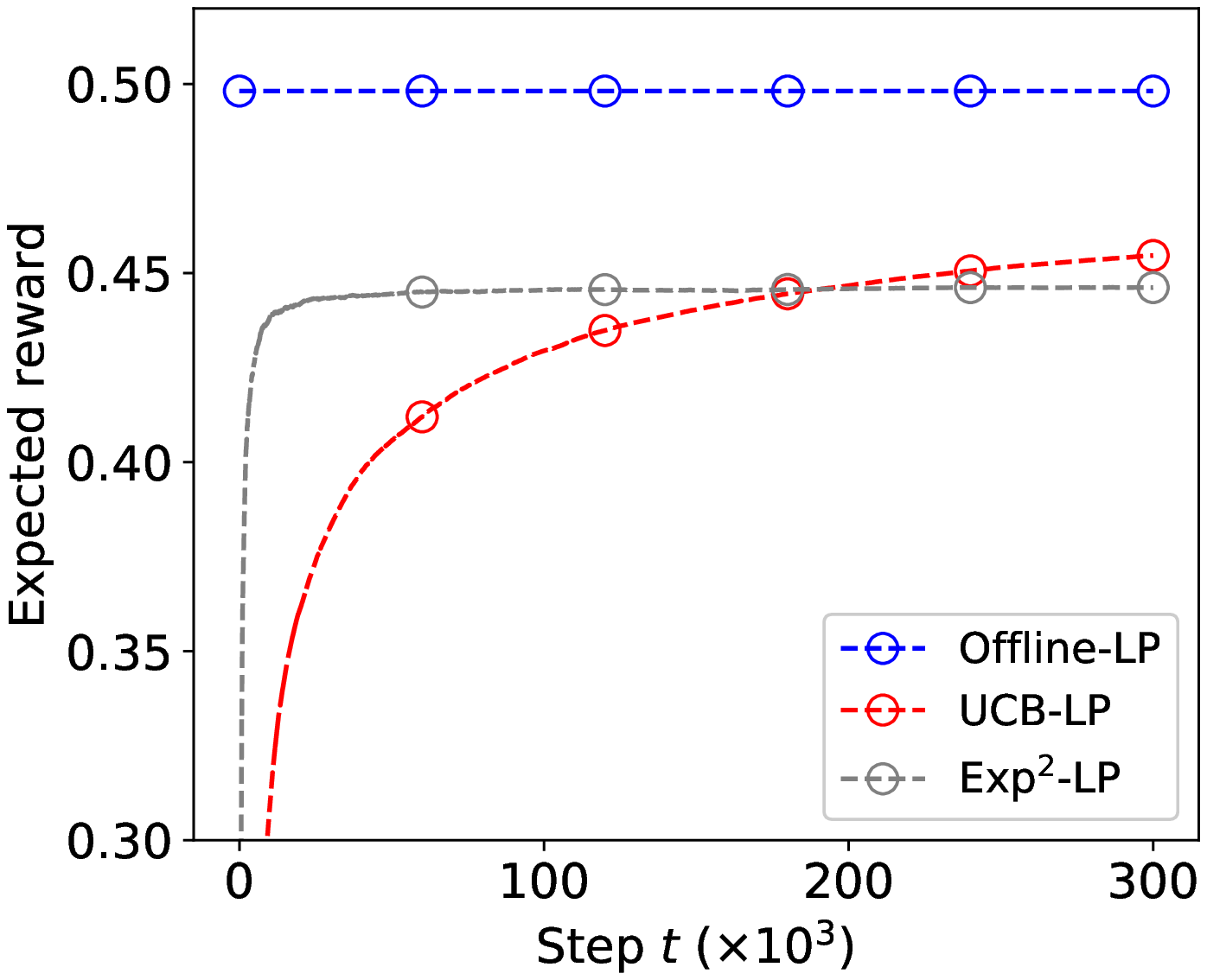} \\
  \hspace{0.18in} (a) \hspace{1.35in} (b) \vspace{-0.05in}
  \caption{Results on synthetic data with a single ad slot. (a) Expected reward of three approaches with \approach{Greedy} oracle. (b) Expected reward of the same approaches with \approach{LP} oracle.} 
  \label{fig:exp_case_1}
\end{figure}

Our results with \approach{Greedy} oracle are reported in \cref{fig:exp_case_1}a. We observe two major trends. 
First, \approach{UCB-Greedy} learns quickly. In particular, its expected reward is around $0.5$ dollars in $30$k steps and exceeds $0.5$ dollars after $50$k steps. We note that \approach{UCB-Greedy} slightly outperforms \approach{Offline-Greedy} after $50$k steps. Indeed, since \approach{Greedy} oracle is not guaranteed to return the optimal solution, it is possible to learn a better approximation online than offline.
Second, \approach{Exp$^2$-Greedy} is consistently worse than \approach{UCB-Greedy} and its expected reward is only $0.44$ dollars in $100$k steps. This shows that $n_0 = 500$ random exploration steps in \approach{Exp$^2$-Greedy} are less statistically efficient than more intelligent continuous exploration in \approach{UCB-Greedy}.

The results with \approach{LP} oracle are reported in \cref{fig:exp_case_1}b. We observe similar trends to those in \cref{fig:exp_case_1}a. One minor difference is that \approach{UCB-LP} performs worse than \approach{Exp$^2$-LP} in the first $150$k steps. However, it outperforms \approach{Exp$^2$-LP} after $200$k steps and its expected reward approaches $0.45$ dollars in $300$k steps. 
The reason \approach{UCB-LP} learns more slowly than \approach{UCB-Greedy} is that the linear program in \approach{UCB-LP} is not sensitive to small perturbations of model dynamics. That is, minor changes in the optimistic estimates of acceptance probabilities do not affect the output of the linear program. Therefore, \approach{UCB-LP} explores all parameters of the model in the descending order of prices, which is inefficient. This is because higher prices are always preferred if the acceptance probabilities at all prices do not differ much. Only when the acceptance probabilities at higher prices become lower, \approach{UCB-LP} explores other lower prices.

\begin{figure}[t]
  \centering
  \includegraphics[height=1.35in]{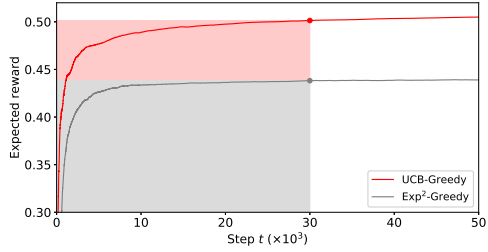}
  \caption{Illustration of publisher insights. The gray rectangle represents the revenue of \approach{Exp$^2$-Greedy} in $30\mathrm{k}$ steps, $0.438 \times 30\mathrm{k} = 13.14\mathrm{k}$ dollars. The red rectangle represents the difference in the revenues of \approach{UCB-Greedy} and \approach{Exp$^2$-Greedy} in $30\mathrm{k}$ steps, $(0.501 - 0.438) \times 30\mathrm{k} = 1.89\mathrm{k}$ dollars.}
  \label{fig:exp_case_1c}
\end{figure}

\subsection{Publisher Insights}
\label{sec:publisher_insights0}

From the perspective of a publisher, our plots of the expected reward in the first $t$ steps can answer the following questions:
(1) What is the revenue of a strategy up to step $t$?
(2) What is the difference in revenues of strategies $A$ and $B$ up to step $t$?

The first question can be answered as follows. The revenue of a strategy up to step $t$ is equal to its expected reward up to step $t$ times $t$. 
In \cref{fig:exp_case_1}a, for instance, the expected reward of \approach{Offline-Greedy} in $100$k steps is $0.504$ dollars. Therefore, the revenue of \approach{Offline-Greedy} in $100$k steps is $50.4$k dollars. 
The expected reward of \approach{UCB-Greedy} in $100$k steps is $0.507$ dollars. Therefore, the revenue of \approach{UCB-Greedy} in $100$k steps is $50.7$k dollars. 
By the same line of reasoning, the revenue of \approach{Exp$^2$-Greedy} in $100$k steps is $44$k dollars.

The second question can be answered as follows. The difference in revenues of strategies $A$ and $B$ up to step $t$ is equal to the difference of their expected rewards up to step $t$ times $t$. We illustrate this in \cref{fig:exp_case_1c}. The expected rewards of \approach{UCB-Greedy} and \approach{Exp$^2$-Greedy} in $30$k steps are $0.501$ and $0.438$ dollars, respectively. 
Therefore, the difference in their expected rewards is $0.063$ dollars, and the difference in their revenues in $30$k steps is $1.89$k dollars.
This increase in revenue is a result of the improved statistical efficiency of \approach{UCB-Greedy} relative to \approach{Exp$^2$-Greedy}.

\subsection{Real Data}

\subsubsection{Selling a Single Ad Slot}
\label{sec:single ad slot}
In this experiment, we show that our algorithm can learn to sell a single ad slot, whose dynamics is estimated from a real-world dataset. 

We experiment with a real-world dataset of Real-Time Bidding (RTB) \emph{iPinYou} \cite{Liao2014}. This dataset contains information regarding bidding on ad slots, such as the identity of the ad slot, the winning advertiser, and the winning price. We treat each \emph{advertiser} as an ad \emph{network}. Perhaps surprisingly, the winning price of any advertiser on any ad slot does not change throughout the dataset. This is common in practice because many advertisers do not behave very strategically.

We estimate the valuations of ad networks as follows. Fix the ad slot. Let $n_a$ be the number of times that advertiser $a$ wins bidding and $v_a$ be its winning price, which does not change throughout the dataset. Then ad network $a$ accepts price $p$, independently of all other ad networks, with probability 
\begin{align}
  \bar{w}(a, p)=
  \begin{cases}
    \displaystyle \mathbb{1}\{p \leq v_a\} \frac{n_a}{\sum_{a'} n_{a'}}\,, & p > 0\,; \\
    1\,, & p = 0\,;
  \end{cases}
  \label{eq:exp_case_2}
\end{align}
where $n_a / \sum_{a'} n_{a'}$ is the frequency with which advertiser $a$ wins bids. Basically, $\bar{w}(a, p)$ is the empirical distribution of the acceptance probability of ad network $a$ when offered price $p$. We assume that the zero price is always accepted. This does not fundamentally change our problem and allows us to avoid boundary cases in our simulations.

We experiment with $20$ most active ad slots in the \emph{iPinYou RTB} dataset, and refer to this subset of data as \emph{Active20}. Specifically, there are nine advertisers bidding on these ad slots. The prices in the dataset are in $[0, 300]$. We divide each price by $330$ in order to normalize all prices to $[0, 1]$. As in \cref{sec:synthetic}, all algorithms operate on $11$ discrete price levels. The only major difference from \cref{sec:synthetic} is that the valuations of ad networks are distributed according to \eqref{eq:exp_case_2}.

\begin{figure}[t]
\centering
  \includegraphics[width=1.546in,height=1.35in]{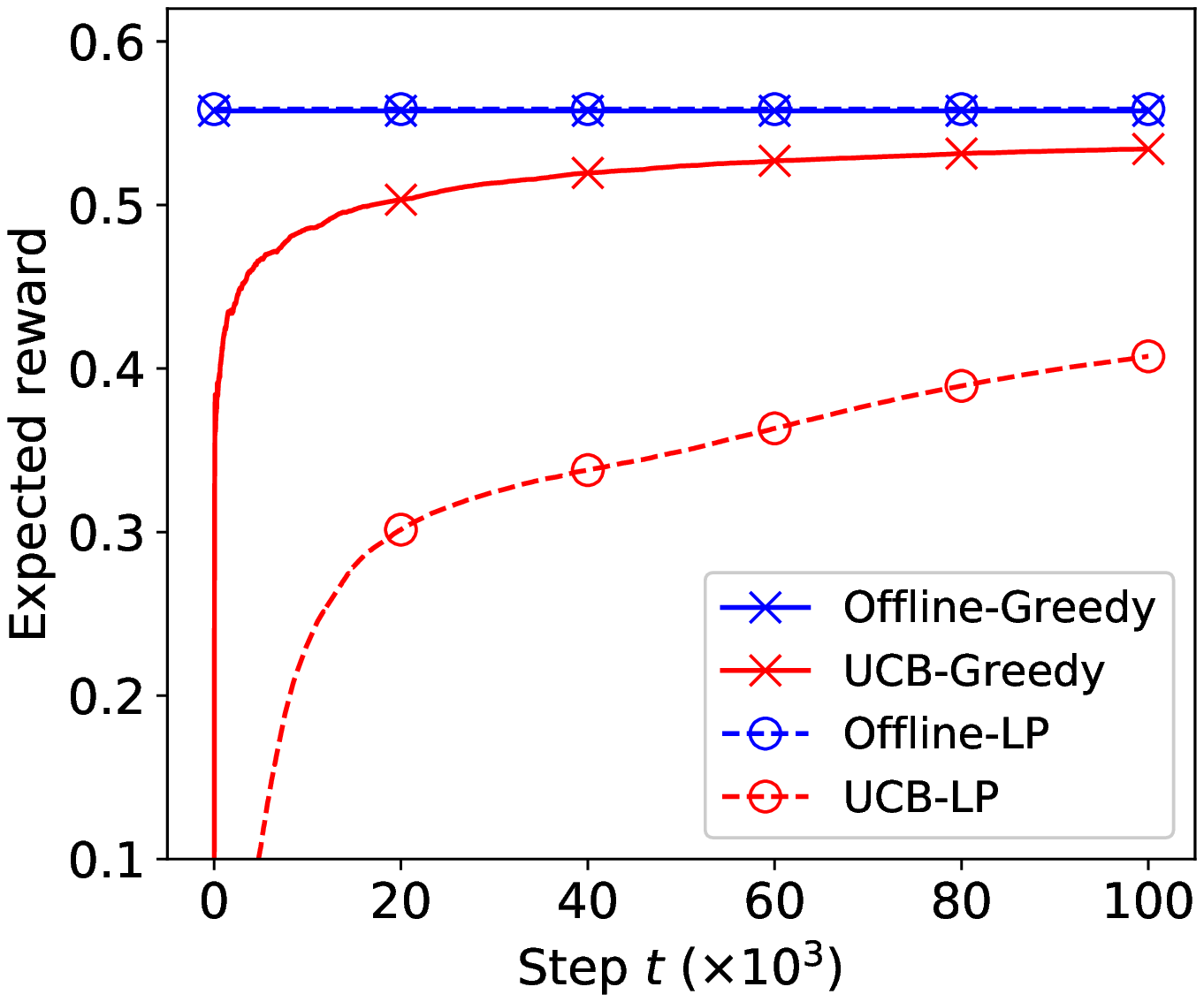}
  \includegraphics[width=1.546in,height=1.35in]{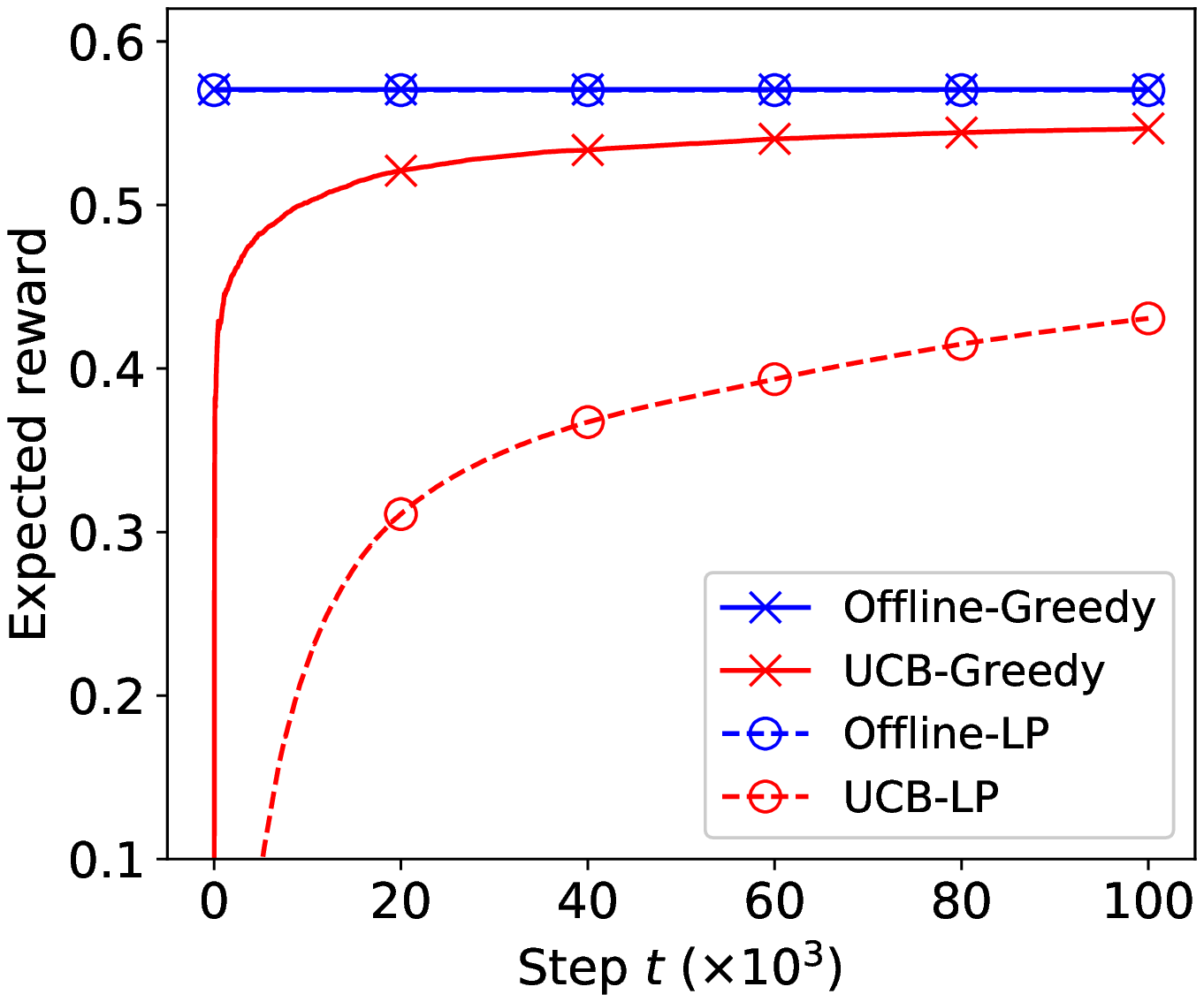} \\
  \hspace{0.18in} (a) \hspace{1.35in} (b) \vspace{-0.05in}
  \caption{Real-world problem with a single ad slot. (a) Expected reward up to step $t$ of \approach{Offline-X} and \approach{UCB-X} on the most active ad slot. (b) Average of the expected rewards of the same approaches over $10$ most active ad slots.}
  \label{fig:exp_case_2}
\end{figure}

Our results on the most active ad slot in \emph{Active20} are reported in \cref{fig:exp_case_2}a. We observe two major trends. 
First, \approach{Offline-X} has the same performance irrespective of the oracle. The expected rewards of both \approach{Offline-Greedy} and \approach{Offline-LP} are $0.56$ dollars in $100$k steps, or equivalently $56$k dollars in revenue. 
Second, \approach{UCB-Greedy} learns faster than \approach{UCB-LP}. In particular, the expected reward of \approach{UCB-Greedy} is $0.53$ dollars in $100$k steps, or equivalently $53$k dollars in revenue. The expected reward of \approach{UCB-LP} is $0.41$ dollars in $100$k steps, or equivalently $41$k dollars in revenue. The difference in the revenues of two approaches in $100$k steps is $12$k dollars.

We also report the average performance of our algorithms on $10$ most active ad slots in \emph{Active20} in Figure \ref{fig:exp_case_2}b. These trends are extremely similar to those in \cref{fig:exp_case_2}a. This experiment validates that our findings from \cref{fig:exp_case_2}a are not limited to the most active ad slot, and that they apply to different ad slots.

\subsubsection{Selling Multiple Ad Slots}
\label{sec:multiple ad slots}
Publishers often have different pages and sell hundreds of ad slots. To facilitate operations and speed up learning, one option is to learn a single selling strategy across multiple ad slots. In this experiment, we evaluate this option. In particular, if the acceptance probabilities of ad networks do not change much with the ad slots, learning of one common model is expected to lead to much faster learning of a near-optimal policy.

The acceptance probabilities of ad networks are estimated in the same way as \cref{sec:single ad slot}. We consider the following model of interaction with multiple ad slots. Let $m$ be the number of ad slots. The ad slot at time $t$ is drawn uniformly at random from these $m$ ad slots. The publisher knows the identity of the ad slot at time $t$ and its goal is to maximize its reward, in expectation over the randomness in the choice of the ad slot at time $t$ and the behavior of ad networks. 
We propose two solutions to this problem. One is \approach{UCB-X} that treats each ad slot separately and computes UCBs for all pairs of ad networks and prices in each ad slot. The other is \approach{Joint-UCB-X} that treats all ad slots as a single slot, and computes UCBs for all pairs of ad networks and prices. \approach{Joint-UCB-X} is expected to perform well if the acceptance probabilities of ad networks do not vary much across ad slots.

\begin{figure}[t]
  \centering
  \includegraphics[width=1.546in,height=1.35in]{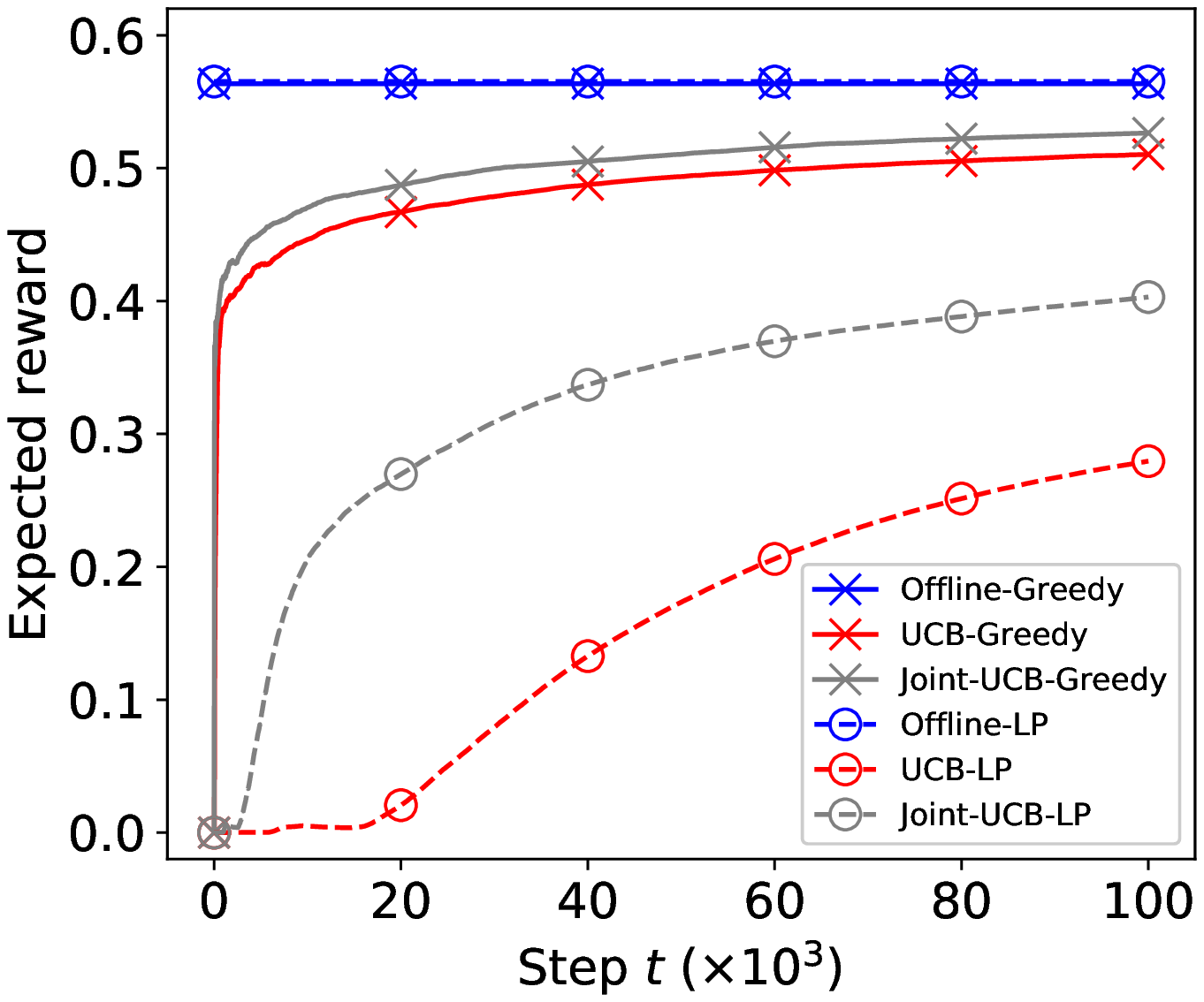}
  \includegraphics[width=1.546in,height=1.35in]{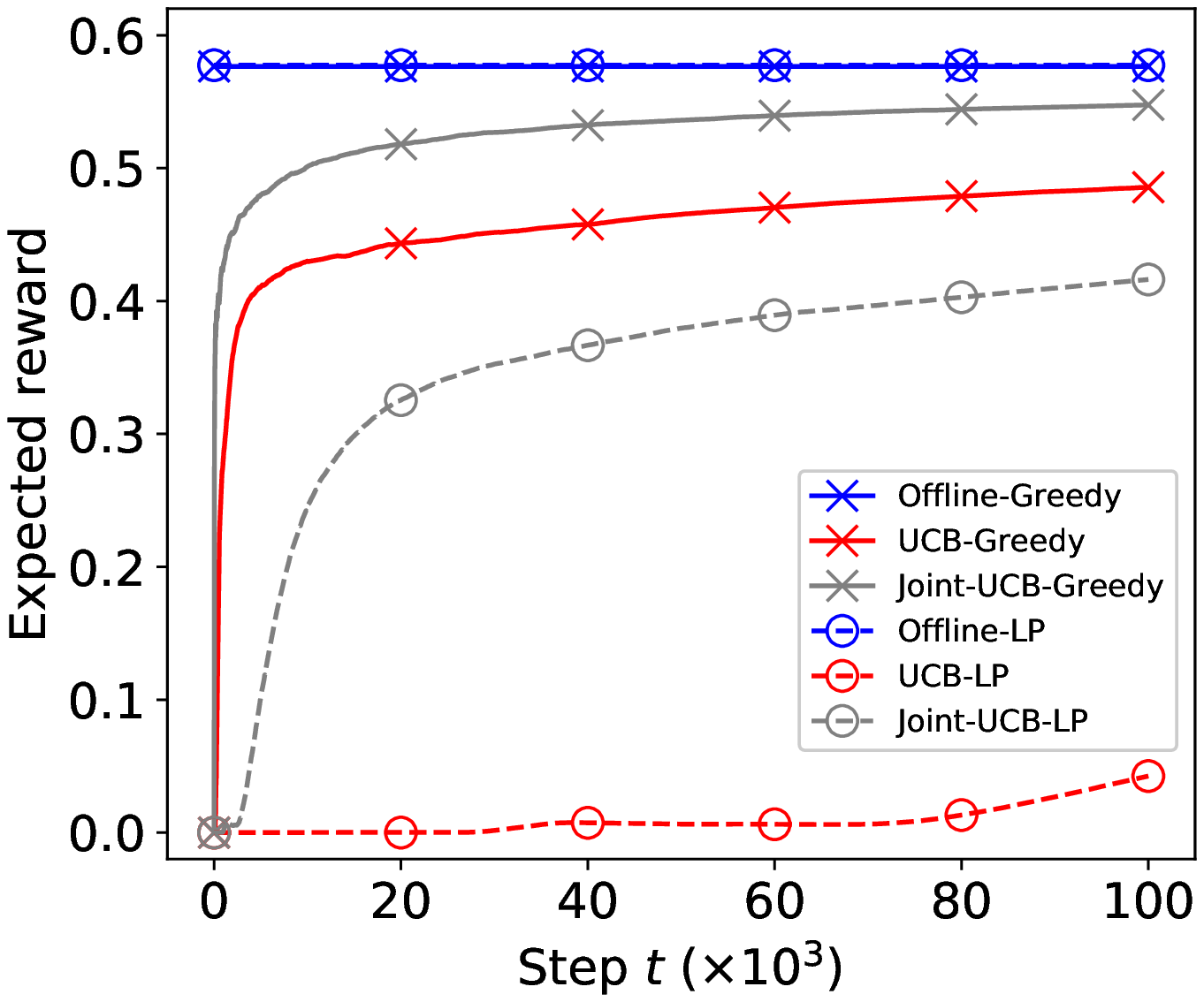} \\
  \hspace{0.18in} (a) \hspace{1.35in} (b) \vspace{-0.05in}
  \caption{Real-world problem with multiple ad slots. (a) Expected reward up to step $t$ of \approach{Offline-X}, \approach{UCB-X}, and \approach{Joint-UCB-X} on five most active ad slots. (b) Expected reward of the same approaches on $20$ most active ad slots.}
  \label{fig:exp_case_3}
\end{figure}

Our results on $m = 5$ most active ad slots in \emph{Active20} are reported in \cref{fig:exp_case_3}a. We observe two major trends. 
First, the expected rewards of both \approach{UCB-X} and \approach{Joint-UCB-X} improve over time. The expected reward of \approach{UCB-Greedy} is $0.5$ dollars in $100$k steps, or equivalently $50$k dollars in revenue. The expected reward of \approach{Joint-UCB-Greedy} is $0.53$ dollars in $100$k steps, or equivalently $53$k dollars in revenue. 
Second, \approach{Joint-UCB-X} learns faster than \approach{UCB-X}. In particular, the difference in the expected rewards of \approach{Joint-UCB-Greedy} and \approach{UCB-Greedy} is $0.03$ dollars in $100$k steps, or equivalently $3$k dollars in revenue. This highlights a common trade-off in learning. Although \approach{Joint-UCB-Greedy} learns only an approximate model, this model is easier to learn in a finite time because it has $m$ times less parameters than \approach{UCB-Greedy}. We observe the same trends with \approach{LP} oracle.

Our results on all the ad slots in \emph{Active20} are reported in \cref{fig:exp_case_3}b. These trends are similar to those in \cref{fig:exp_case_3}a. We note that the benefits of \approach{Joint-UCB-Greedy} and \approach{Joint-UCB-LP} increase with $m$.

\subsubsection{Selling to Aggregated Ad Networks}
\label{sec:ad networks}
A common scenario is that publishers interact with third parties, which aggregate multiple ad networks. 
In this section, we study the impact of ad network aggregation on learning publisher revenue.

The third parties are modeled as follows. All ad networks are partitioned into $h$ groups, $G_1,\ldots,G_{h}$. The values for $h$ will be specified later. When price $p$ is offered to group $G_i$, any ad network $a \in G_i$ accepts the offered price with probability $\bar{w}(a, p)$ in \eqref{eq:exp_case_2}, independently of all other ad networks. If at least one $a \in G_i$ accepts, $G_i$ accepts. From the point of view of the publisher and our algorithms, each group is treated as an ad network.

\vspace{1.5mm}
\noindent {\bf Learning with Aggregated Ad Networks. \hspace{1mm}} 
We first show that our algorithm can learn to sell to aggregated ad networks. We also show that \approach{LP} oracle leads to faster learning than \approach{Greedy} oracle when the dynamics of selling is more complicated.

We set $h=2$ and evaluate \approach{Offline-X} and \approach{UCB-X} on the most active ad slot in \emph{Active20} dataset under two settings. In the first experiment, we fix six ad networks in $G_1$ and put the remaining three ad networks in $G_2$. In the second experiment, we put six random ad networks in $G_1$ and the remaining three ad networks in $G_2$. This experiment is repeated with $10$ random partitions.

\begin{figure}[t]
  \centering
  \includegraphics[width=1.546in,height=1.35in]{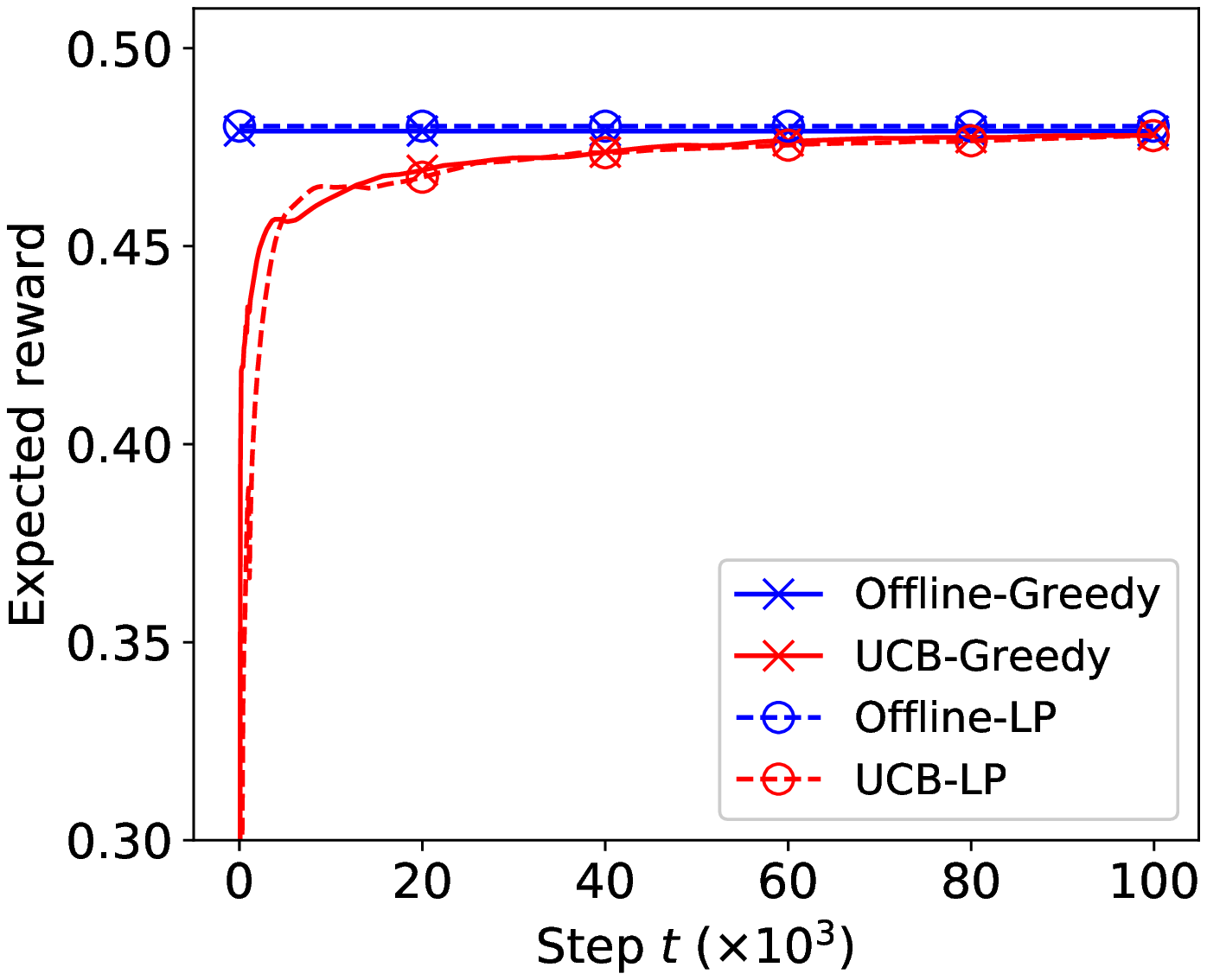}
  \includegraphics[width=1.546in,height=1.35in]{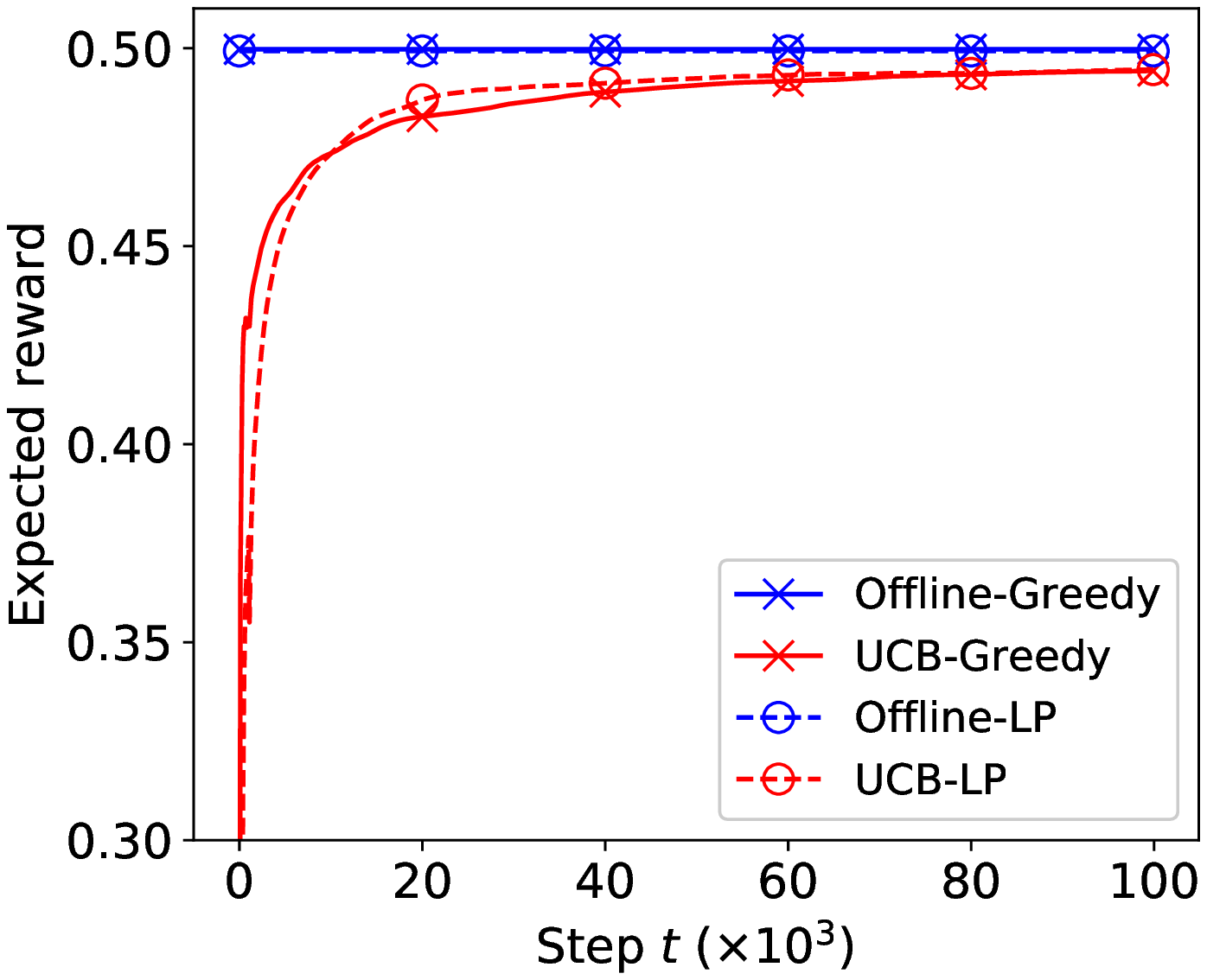} \\
  \hspace{0.18in} (a) \hspace{1.35in} (b) \vspace{-0.05in}
  \caption{Real-world problem with aggregated ad networks. (a) Expected reward up to step $t$ of \approach{Offline-X} and \approach{UCB-X} over the fixed partition of the most active ad slot. (b) Average of the expected rewards of the same approaches over $10$ random partitions.}
  \label{fig:exp_case_4}
\end{figure}

The results of the first experiment are shown in \cref{fig:exp_case_4}a. We observe one major trend. The expected reward of \approach{UCB-X} converges to that of the best approximation in hindsight irrespective of the oracle. 
For example, the expected reward of \approach{Offline-X} is around $0.48$ dollars in $100$k steps, or equivalently $48$k dollars in revenue. The expected reward of \approach{UCB-X} reaches almost $0.478$ dollars in $100$k steps, or equivalently $47.8$k dollars in revenue. The difference in revenues is merely $200$ dollars, which indicates that \approach{UCB-X} can learn a very good approximation in this experiment.

The results of the second experiment are shown in \cref{fig:exp_case_4}b. We make two additional observations. First, the trends are very similar to those in \cref{fig:exp_case_4}a. This shows that our algorithm \approach{UCB-X} does not overfit to a specific group of ad networks. Second, algorithms with \approach{LP} oracle learn slightly faster than those with \approach{Greedy} oracle. For example, the expected reward of \approach{UCB-Greedy} and \approach{UCB-LP} are respectively $0.491$ and $0.493$ dollars in $60$k steps. The difference of the expected rewards is $0.002$ dollars in $60$k steps, or equivalently $120$ dollars in revenue.

\vspace{1.5mm}
\noindent {\bf Publisher revenue with Aggregated Ad Networks. \hspace{1mm}} 
Finally, we study the impact of ad network aggregation on the expected revenue of publisher.

Again, we evaluate \approach{Offline-X} and \approach{UCB-X} on the most active ad slot in \emph{Active20} dataset but under three different configurations:
\begin{enumerate}
 \item Configuration 1: $h=2$ with group sizes of six and three.
 \item Configuration 2: $h=3$ with group sizes of four, four and one.
 \item Configuration 3: $h=9$ where all group sizes are one.
\end{enumerate}
These configurations represent different degrees of ad network aggregation. In all the configurations, the ad networks are partitioned in a uniformly random fashion. Each configuration is repeated for $10$ times.

\begin{figure}[t]
  \centering
  \includegraphics[width=1.546in,height=1.35in]{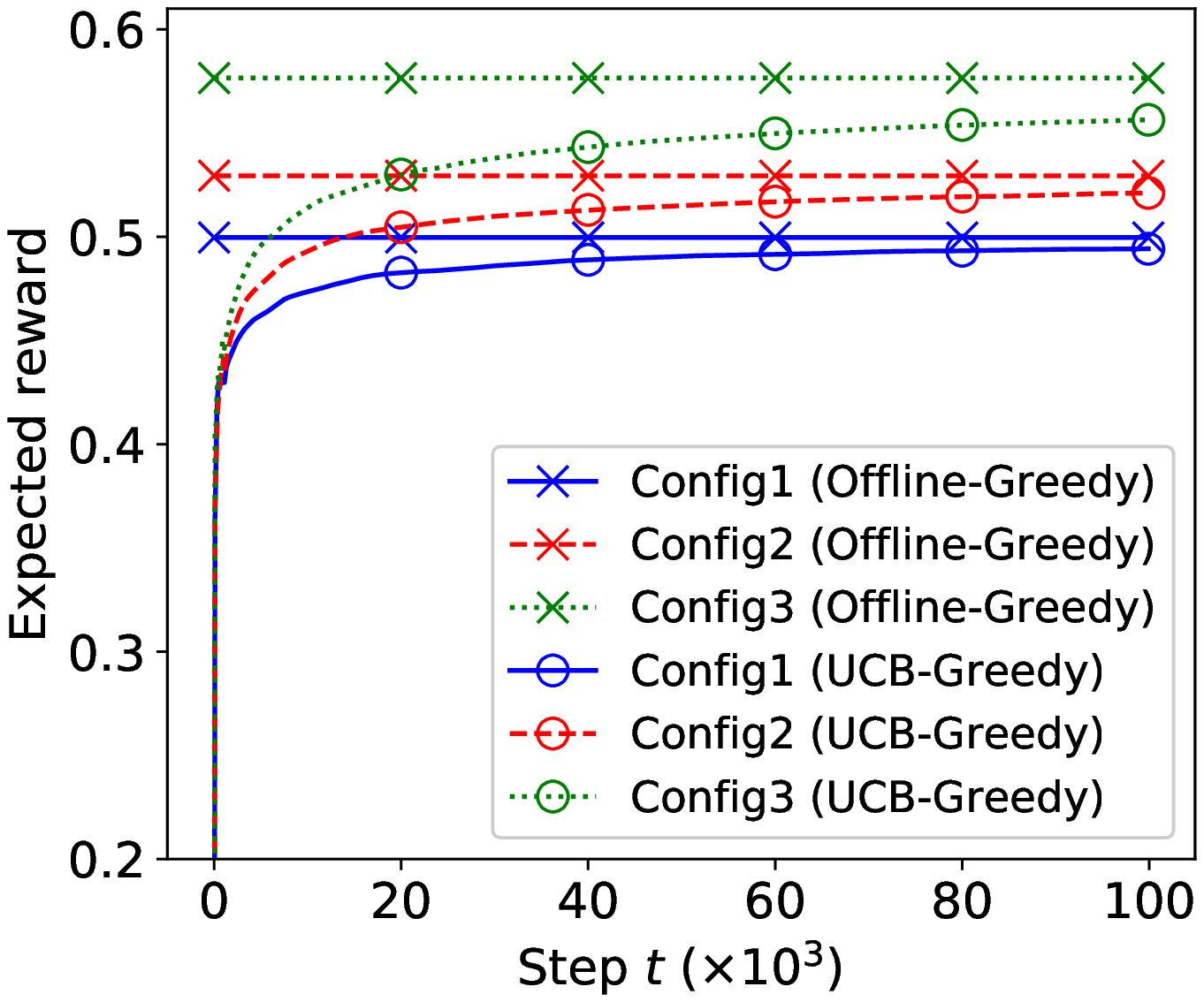}
  \includegraphics[width=1.546in,height=1.35in]{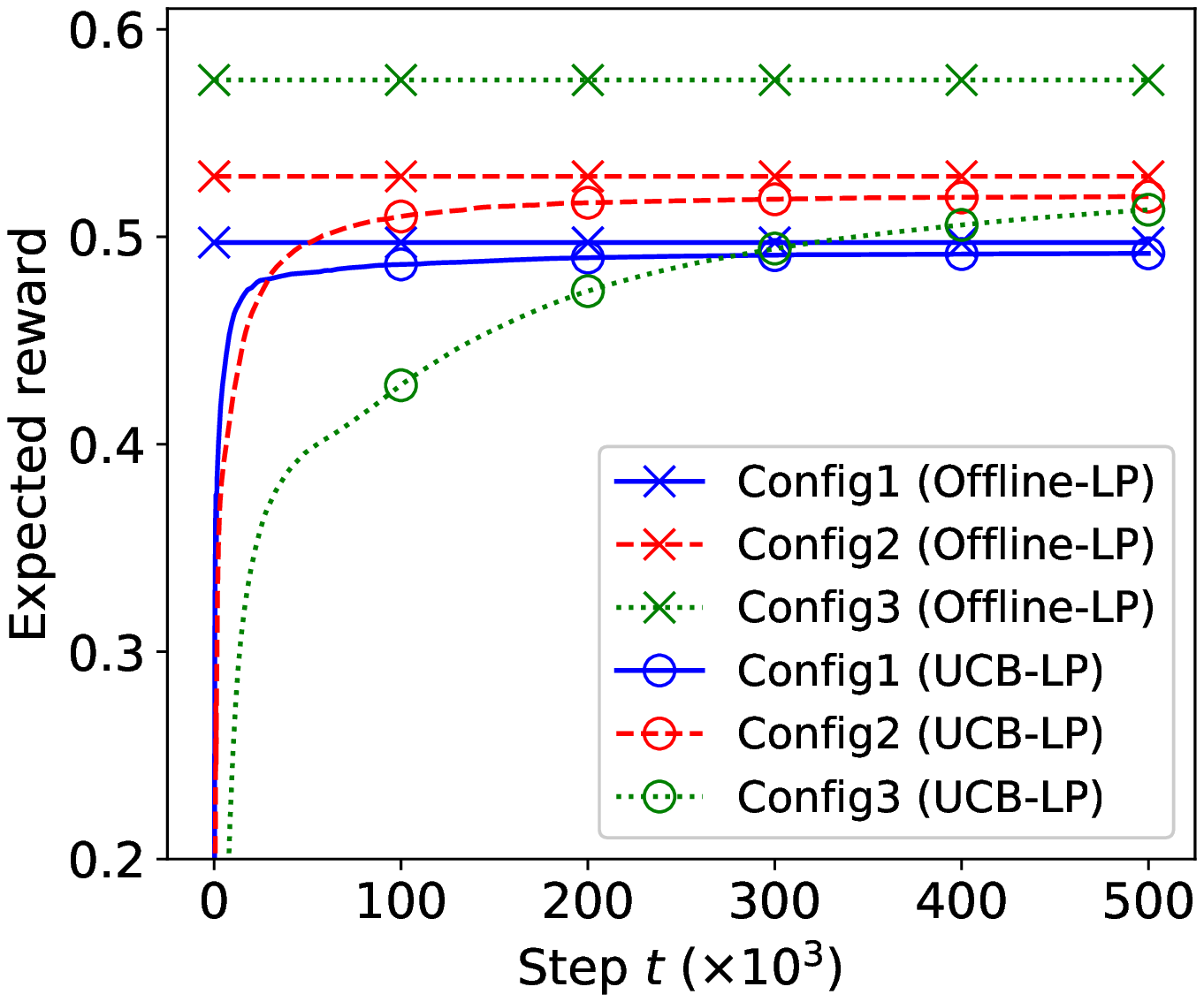} \\
  \hspace{0.18in} (a) \hspace{1.35in} (b) \vspace{-0.05in}
  \caption{Impact of three configurations of aggregated ad networks on expected revenues. (a) Publisher expected reward up to step $t$ for oracle \approach{Greedy}. (b) Publisher expected reward up to step $t$ for oracle \approach{LP}.}
  \label{fig:exp_case_5}
\end{figure}

The results of oracles \approach{Greedy} and \approach{LP} are respectively reported in \cref{fig:exp_case_5}a and \cref{fig:exp_case_5}b. 
We observe the similar results to the previous experiment that our algorithms can learn to sell under all the configurations of ad network aggregations.
Moreover, we observe two additional interesting trends.

First, less aggregation of ad networks results in higher expected reward. 
Take the oracle \approach{Greedy} as example. 
As shown in \cref{fig:exp_case_5}a, the expected rewards of \approach{Offline-Greedy} and \approach{UCB-Greedy} under Configuration 3 are respectively $0.576$ and $0.556$ dollars in $100$k steps. 
They are both higher than the expected rewards acquired from other configurations where ad networks aggregate into groups.
One explanation is that less aggregation of ad networks allows the publisher to better customize prices to ad networks, and hence the expected reward is higher.

Second, less aggregation of ad networks requires longer time to find the optimal solution, especially for the algorithm \approach{UCB-LP}.
To illustrate this phenomenon, we run all algorithms with oracle \approach{LP} for more steps ($t=500$k). 
As shown in \cref{fig:exp_case_5}b, the expected reward of \approach{UCB-LP}  reaches $0.513$ dollars in $500$k steps when there are nine individual ad networks (Configuration 3).
It exceeds the expected reward of $0.492$ dollars in the case of two aggreated groups (Configuration 1) and is close to $0.519$ dollars of three groups (Configuration 2).
With less aggregation, although our algorithm statistically should be able to collect more responses per waterfall run, 
it needs to learn the behavior of more groups. 


\section{Related Work}
\label{sec:related work}

Our work is at the intersection of online advertising and online learning with partial feedback.

The problem of waterfall optimization was studied before under the name of ``sequential posted price mechanisms'' \cite{Chawla2010,chak2010,Adamczyk2017,Babaioff2015,Kleinberg2003}. In \cite{Chawla2010,chak2010,Adamczyk2017}, the acceptance probabilities of ad networks are assumed to be known by the publisher. \cite{Babaioff2015,Kleinberg2003} study the waterfall optimization problem in an online setting, under the assumption that all ad networks have the same acceptance probabilities. We do not make any of these assumptions.

Our work is a generalization of online learning to rank in the cascade model \cite{kveton15cascading,kveton15combinatorial}. More specifically, cascading bandits can be viewed as waterfall bandits when $Q = \{1\}$. This seemingly minor change has major implications. For instance, when $Q = \{1\}$, the optimal solution in \eqref{eq:waterfall_optimization} can be computed greedily. In our case, no polynomial-time algorithm is known for solving \eqref{eq:waterfall_optimization}. From the learning point of view, we learn $K M$ statistics. In cascading bandits, only $K$ statistics are learned because $M = |Q| = 1$. 

Our problem is a form of partial monitoring \cite{bartok12adaptive,agrawal89asymptotically}, which is a harder class of learning problems than multi-armed bandits. The general algorithms in partial monitoring cannot solve our problem computationally efficiently because their computational cost is $\Omega(|\mathcal{A}|)$, where $|\mathcal{A}|$ is exponential in the number of ad networks.

Our setting is also reminiscent of stochastic combinatorial semi-bandits \cite{gai12combinatorial,chen14combinatorial,kveton15tight,wen15efficient}, which can be solved statistically efficiently by UCB-like algorithms. The difference is that our feedback is less than semi-bandit. In particular, if an ad network accepts an offer, the learning agent does not learn if any of the subsequent ad networks would have accepted their offered prices. In combinatorial semi-bandits, all of these events are assumed to be observed. Therefore, our problem cannot be solved as a combinatorial semi-bandit.

\section{Conclusions}
\label{sec:conclusions}

For the waterfall, we propose the algorithm $\alg$, a computationally and sample efficient online algorithm for learning to price, which maximizes the expected revenue of the publisher. We derive a sublinear upper bound on the $n$-step regret of $\alg$. Note that $\alg$ solves a general problem of learning to maximize \eqref{eq:waterfall_optimization} from partial feedback. Therefore, although our main focus is online advertising, the algorithm may have other applications, especially in learning to price.

We evaluate $\alg$ on both synthetic and real-world data, and show that it quickly learns competitive strategies to the best approximations in hindsight. In addition, we investigate multiple real-world scenarios that are of a particular interest of publishers. We show that $\alg$ can learn to sell in these scenarios and it does not overfit. 

We leave open several questions of interest. For instance, note that the update of statistics in $\alg$ can be easily modified to leverage the following two monotonicity properties. When ad network $a$ accepts price $p$, it would have accepted any lower price $p' < p$. Similarly, when ad network $a$ does not accept price $p$, it would have not accepted any higher price $p' > p$. Roughly speaking, this would make $\alg$ more statistically efficient. However, it is non-trivial to prove that this would result in a better regret bound than that in \cref{sec:analysis}. 
We leave these for future work.

\bibliographystyle{plain}
\bibliography{Papers,References}

\clearpage
\onecolumn
\appendix

\section{Appendix: Proof of Theorem \ref{thm::thm1}}\label{subsub:thm_proof}

We first prove that the function $f(A, u)$ is monotone in the weight function $u$, for any fixed action $A$.

 \begin{restatable}{lemma}{lemma_mon}
  Consider $\bar{w},\bar{v} \in [0,1]^{[K] \times Q}$ such that for all $i \in [K], j \in Q$, $\bar{w}(i,j) \leq \bar{v}(i,j) $.
  If the items of action $A$ are sorted in descending order of prices, then 
  \begin{equation}
  \label{eqn:monotonicity}
    f(A,\bar{w}) \leq f(A,\bar{v})
  \end{equation}
  \label{thm::prop1}
 \end{restatable}

\begin{proof}
We prove Lemma~\ref{thm::prop1} based on the mathematical induction on $K$, the number of ad networks. 
 
 \noindent \textbf{Induction base:} We first prove that equation~\ref{eqn:monotonicity} holds for the case with $K=1$. Notice that for $K=1$,
 \begin{equation}
  f(A,\bar{w}) = p_1 \bar{w}(a_1,p_1) \leq p_1 \bar{v}(a_1,p_1) = f(A,\bar{v})
 \end{equation}

\noindent \textbf{Induction step:} For any integer $m \geq 1$, we then prove that if equation~\ref{eqn:monotonicity} holds for $K=m$, then it also holds for
$K=m+1$. Recall that for $K=m+1$, $A= \left( (a_{1}, p_{1}), \ldots, (a_{m+1}, p_{m+1}) \right)$. 
To simplify the exposition, we also define the term $\tilde{A}= \left( (a_{2}, p_{2}), \ldots, (a_{m+1}, p_{m+1}) \right)$.  
Notice that

\begin{align}
            f(A,\bar{w})  = & \bar{w}(a_1,p_1) p_1 + (1 - \bar{w}(a_1,p_1)) f(\tilde{A},\bar{w}) \nonumber \\
 \stackrel{(a)}{\leq} & \bar{w}(a_1,p_1) p_1 + (1 - \bar{w}(a_1,p_1)) f(\tilde{A},\bar{v})  \nonumber
\end{align}
where $(a)$ follows the induction hypothesis.
Moreover, since  $f(\tilde{A},\bar{v})$ is the expected revenue of $\tilde{A}$, $f(\tilde{A},\bar{v}) \leq p_2 \leq p_1$. Therefore,
\begin{align}
  & 0 \leq (p_1   - f(\tilde{A},\bar{v})) (\bar{v}(a_1,p_1) - \bar{w}(a_1,p_1))  
\end{align}
which implies
\begin{align*}
   \, \bar{w}(a_1,p_1) p_1 + (1 - \bar{w}(a_1,p_1)) f(\tilde{A},\bar{v}) 
 \leq  \, \bar{v}(a_1,p_1) p_1 + (1 - \bar{v}(a_1,p_1)) f(\tilde{A},\bar{v})  
 =   \, f(A,\bar{v}) 
\end{align*}

\noindent As a result, $f(A,\bar{w}) \leq f(A,\bar{v})$. 
\end{proof}

We now prove Theorem~\ref{thm::thm1}.
First, we define the ``bad event" $\mathcal{E}_t$ at time $t$ as the event that at least one $\bar{w}(e)$ is outside its confidence interval at time $t$, 
\begin{align}
  \mathcal{E}_t = \{ \exists e \in E \; s.t. \; |\bar{w}(e) - \hat{w}_{T_{t-1}(e)}(e)| > c_{t-1, T_{t-1}(e)}  \} \, .
\end{align}

Notice that $\bar{\mathcal{E}}_t$, the complement of $\mathcal{E}_t$, is considered as the ``good event" at time $t$.
Similar to \cite{kveton15cascading}, 
we define the event $G_{e,t}$ as the event that item $e=(a,p)$ is ``observed" at time $t$ (i.e. ad network $a$ is called and offered price $p$ at time $t$):
\begin{align}
   G_{e,t} =& \{ \exists 1 \leq k \leq K \ni \; d_{t,k} = e ,  B_t \geq k \}
\end{align}
where $d_{t,i} = (a_{t,i},p_{t,i})$.  
In addition, we define $\mathbb{E}_t[\cdot] = \mathbb{E}[\cdot| \mathbb{H}_t]$ where $\mathbb{H}_t$ is the history of all
actions and feedbacks until time $t-1$ plus the action $A_t$, which is determined by $\mathbb{H}_t$ under Algorithm~\ref{alg:alg1}.
The following lemma bounds the per-step scaled regret $R_t^\gamma = \gamma f(A^*, \bar{w}) - f(A_t, \bar{w})$ under the ``good event" $\bar{\mathcal{E}}_t$:

\lemmadecomp

\begin{proof} 
Conditioning on the event $\bar{\mathcal{E}}_t$, we have $\bar{w} \leq U_t$.
Lemma~\ref{thm::prop1} states that $\bar{w} \leq U_t$ implies
$f(A, \bar{w}) \leq f(A, U_t)$. Then we have the following bound on $R_t^\gamma$:
\begin{align}
R_t^\gamma \stackrel{(a)}{=} & \gamma f(A^*, \bar{w}) - f(A_t, \bar{w}) \nonumber \\
\stackrel{(b)}{\leq} & \gamma f(A^*, U_t ) - f(A_t, \bar{w}) \nonumber \\
\leq  & \gamma \max_A f(A, U_t ) - f(A_t, \bar{w}) \nonumber \\
\stackrel{(c)}{\leq} & f(A_t, U_t) - f(A_t, \bar{w}),
\end{align}
where (a) follows from the definition of $R_t^\gamma$, 
(b) follows from Lemma~\ref{thm::prop1}, and (c) follows from the fact that 
$A_t$ is computed from a $\gamma$-approximation algorithm.

To simplify the exposition, in the rest of this proof, we use $\bar{w}_i$ and $U_i$ to respectively denote $\bar{w}(a_{t,i}, p_{t,i})$
and $U_t (a_{t,i}, p_{t,i})$, and use $p_i$ to denote $p_{t,i}$. Then we have
\begin{align}
R_t^\gamma \stackrel{(a)}{\leq} & \,
 \sum_{i=1}^K \left[ \prod_{j<i} [1 - U_j ]\right] U_i p_{i} -
\sum_{i=1}^K \left[ \prod_{j<i} [1 - \bar{w}_j ]\right] \bar{w}_i p_{i} \nonumber \\
 \stackrel{(b)}{\leq} & \,
  \sum_{i=1}^K \left[ \prod_{j<i} [1 - \bar{w}_j ]\right] U_i p_{i} -
\sum_{i=1}^K \left[ \prod_{j<i} [1 - \bar{w}_j ]\right] \bar{w}_i p_{i} \nonumber \\
\stackrel{(c)}{=} & \, \sum_{i=1}^K \left[ \prod_{j<i} [1 - \bar{w}_j ]\right] \left( U_i -\bar{w}_i \right) p_{i} \nonumber \\
\stackrel{(d)}{\leq} & \, \sum_{i=1}^K \left[ \prod_{j<i} [1 - \bar{w}_j ]\right] \left( U_i -\bar{w}_i \right) , \label{eqn2}
\end{align}
where (b) follows from $\bar{w}_i \leq U_i$ for all $i$ under event $\bar{\mathcal{E}}_t$, and (d) follows from $0< p_i \leq 1$ for all 
$i$. Notice that $U_i -\bar{w}_i $ is the ``item-wise" difference between the upper confidence and the mean, and $\prod_{j<i} [1 - \bar{w}_j ]$ is the conditional probability that the $i$th ad network will be called.
We have
\begin{align}
 \prod_{j<i} [1 - \bar{w}_j ] = \mathbb{E}_t \left\{ \mathbb{1}\{ G_{d_{t,i},t}\} \right\} \, .
\end{align}
From (\ref{eqn2}), we have
\begin{align}
  R^{\gamma}_t \leq \sum_{i=1}^K \mathbb{E}_t \left\{ \mathbb{1}\{ G_{d_{t,i},t}\} \right\} \phi_{(a_{t,i},p_{t,i}),t} \, .
  \label{eqn3}
 \end{align} 
\end{proof}

\noindent We use Lemma \ref{lemma::lemma1} to bound $R^{\gamma}(n)$ as follows. 
Notice that 
  $R^{\gamma}(n) = \mathbb{E}[\sum_{t=1}^n \mathbb{1}\{\mathcal{E}_t\} R^{\gamma}_t] + \mathbb{E}[\sum_{t=1}^n \mathbb{1}\{\bar{\mathcal{E}}_t\} R^{\gamma}_t]$.
We use $e$ to refer to an item in $E$. 
As discussed, all prices are less or equal to 1; so, $0 \leq f(A,w) \leq 1$. Hence, we have $R^{\gamma}_t \leq \gamma$. As a result, 
\begin{align*}
       & \mathbb{E}[\sum_{t=1}^n \mathbb{1}\{\mathcal{E}_t\} R^{\gamma}_t] \leq \gamma \mathbb{E}[\sum_{t=1}^n \mathbb{1}\{\mathcal{E}_t\}] \nonumber \\
  \leq &  \gamma \sum_{e \in E} \sum_{t=1}^n \sum_{s=1}^t \mathbb{P}(|\bar{w}(e) - \hat{w}_s(e)|\geq c_{t,s}) \nonumber \\
  \stackrel{(a)}{\leq} & 2\gamma \sum_{e \in E} \sum_{t=1}^n \sum_{s=1}^t \exp(-3\log t) \nonumber \\
  \leq & 2\gamma \sum_{e \in E} \sum_{t=1}^n t^{-2} \leq \gamma \frac{\pi^2}{3} |E| =\gamma \frac{\pi^2}{3} MK
\end{align*}
In the above derivation, $(a)$ follows Hoeffding's inequality.
Notice that for all $e \in E$ and $t \leq n$, we have
(1) $\phi_{e,t}  \leq 2c_{t-1,T_{t-1}(e)}$ under event $\bar{\mathcal{E}}_t$, (2) $c_{t-1,T_{t-1}(e)}  \leq c_{n,T_{t-1}(e)}$.
Based on Lemma \ref{lemma::lemma1}, we have
\begin{align*}
  & \mathbb{E} \left[\sum_{t=1}^n \mathbb{1}\{\bar{\mathcal{E}}_t\} R^{\gamma}_t \right] \leq \sum_{e \in E} \mathbb{E} \left[  \sum_{t=1}^n  \mathbb{1}\{\bar{\mathcal{E}}_t, G_{e,t}\} \phi_{e,t} \right]\\
  \leq & \; 2 \sum_{e \in E} \mathbb{E} \left[ \sum_{t=1}^n   \mathbb{1}\{\bar{\mathcal{E}}_t, G_{e,t}\} c_{t-1,T_{t-1}(e)} \right]
  \leq  \; 2\sqrt{1.5 \log n} \sum_{e \in E} \mathbb{E} \left[ \sum_{t=1}^n  \mathbb{1}\{\bar{\mathcal{E}}_t, G_{e,t}\} \sqrt{\frac{1}{T_{t-1}(e)}} \right]
\end{align*}
Notice that $T_{t-1}(e) \leq n$ and once $e$ is observed, $T_t(e)$ is increased by $1$, thus we have
\begin{align*}
  \mathbb{E} [ \sum_{t=1}^n   \mathbb{1}\{\bar{\mathcal{E}}_t, G_{e,t}\} \sqrt{\frac{1}{T_{t-1}(e)}} ] 
 \leq  \sum_{t=1}^n \frac{1}{\sqrt{t}} \leq 1+ \int_{1}^{n} \frac{1}{\sqrt{t}}dt \leq 2\sqrt{n}-1 < 2\sqrt{n}
\end{align*}
Recall that $|E|=MK$, we have
\begin{align*}
 \mathbb{E}[\sum_{t=1}^n \mathbb{1}\{\bar{\mathcal{E}}_t\} R^{\gamma}_t] \leq 4 MK \sqrt{1.5 n \log n}
\end{align*}
This concludes the proof for Theorem~\ref{thm::thm1}.

\end{document}